\documentclass[letterpaper]{article} 
\usepackage[preprint]{aaai2027}
\usepackage[hyphens]{url}  
\usepackage{graphicx} 
\usepackage{natbib}  
\usepackage{caption} 
\usepackage{algorithm}
\usepackage{algorithmic}

\usepackage{newfloat}
\usepackage{listings}
\DeclareCaptionStyle{ruled}{labelfont=normalfont,labelsep=colon,strut=off} 
\floatstyle{ruled}
\newfloat{listing}{tb}{lst}{}
\floatname{listing}{Listing}

\usepackage{enumitem}
\usepackage[utf8]{inputenc} 
\usepackage[T1]{fontenc}    
\usepackage{url}            
\usepackage{booktabs}       
\usepackage{nicefrac}       
\usepackage{microtype}      
\usepackage{xcolor}         
\usepackage{times}
\usepackage{epsfig}
\usepackage{graphicx}

\usepackage{amssymb}
\usepackage{tabularx}
\usepackage{microtype}
\usepackage{graphicx}
\usepackage{subcaption}
\usepackage{booktabs} 

\usepackage{makecell}
\usepackage{ifthen}
\usepackage{xspace}
\usepackage{upgreek}
\usepackage{pifont}
\usepackage{colortbl}
\usepackage{multirow}

\newcommand{\ours}{\texttt{scBatchProx}\xspace}
\usepackage{amsmath, amssymb, amsthm, mathtools}

\theoremstyle{plain}
\newtheorem{theorem}{Theorem}[section]
\newtheorem{proposition}[theorem]{Proposition}
\newtheorem{lemma}[theorem]{Lemma}

\theoremstyle{definition}
\newtheorem{definition}[theorem]{Definition}

\theoremstyle{remark}

\title{scBatchProx: Federated-Inspired Refinement for Stable Cell-Type Discriminability under Heterogeneous Batch Compositions}
\author{
    Quang-Huy Nguyen, Jiaqi Wang\corresponding, Wei-Shinn Ku\corresponding
}
\affiliations{
    Department of Computer Science and Software Engineering\\
    Auburn University\\
    Auburn, AL 36849 USA\\
    \{hqn0001, jqwang, wzk0004\}@auburn.edu
}

\begin{document}

\maketitle

\begin{abstract}
Single-cell integration workflows often construct low-dimensional cell embeddings and then refine them with post-hoc methods to reduce batch effects. This refinement process can become unstable when cell-type compositions vary across batches, with some populations underrepresented or absent in particular batches. The problem becomes more consequential in dynamic single-cell data systems, where newly acquired batches can change both technical conditions and cell-type composition. Such instability can reduce downstream cell-type classification performance and weaken stability under imbalance perturbations.
We introduce \ours, a lightweight post-hoc refinement method for stabilizing single-cell latent embeddings in these heterogeneous and evolving settings. \ours operates directly on precomputed embeddings and treats each batch or study as a client in a federated-inspired optimization procedure. A batch-conditioned FiLM adapter learns local latent updates, while proximal and identity-preserving regularization keep these updates conservative.
Experiments on multi-batch and cross-study single-cell datasets show that \ours improves downstream cell-type classification across different upstream embeddings. In controlled imbalance perturbations, \ours maintains more stable affected-cell-type F1 when selected populations are downsampled or ablated from one batch. In cumulative retraining and continual integration settings, \ours remains effective as new datasets arrive over time. Together, these results suggest that conservative, federated-inspired refinement can help maintain stable single-cell embeddings as batch compositions change across datasets and over time.
\end{abstract}


\section{Introduction}

Single-cell RNA sequencing now profiles thousands to millions of cells in a single experiment~\cite{klein2015droplet,macosko2015highly,satpathy2019massively,datlinger2021ultra}. As these datasets accumulate across studies, platforms, and laboratories, single-cell analysis increasingly depends on integrating cells into shared low-dimensional representations. This integration is difficult because technical variation can separate cells by batch rather than by biological state. If left uncorrected, these batch effects can dominate learned embeddings and distort downstream analyses.


Heterogeneous cell-type composition complicates the interpretation of integrated embeddings. Newly generated datasets rarely contain the same populations in the same proportions. Some cell types may be enriched in one batch, depleted in another, or absent altogether. We refer to this variation as heterogeneous cell-type support. Uneven cell-type support can distort the representation of underrepresented populations and reduce their separability in the integrated embedding. These representation errors can consequently affect downstream cell-type classification~\cite{maan2024characterizing}.


Post-hoc refinement can improve existing embeddings without retraining the upstream embedding construction method. However, strong batch mixing and downstream discriminability are not the same objective. A refinement procedure can improve batch alignment while introducing latent changes that are unnecessary for preserving cell-type discriminability. The problem becomes more difficult as datasets evolve. Newly arriving batches can change both technical variation and cell-type composition, so repeated global correction may move previously stable embeddings. These settings call for a refinement strategy that adapts locally to each batch while limiting unnecessary movement in the latent space.


Federated learning (FL) provides a useful perspective for this problem. FL systems are designed to optimize across heterogeneous clients, each with its own local data distribution, while maintaining a shared global model. This structure closely mirrors single-cell integration under heterogeneous batch compositions, where each batch exhibits its own cell-type distribution and technical characteristics. It also naturally supports dataset evolution, since newly arriving clients can participate in optimization without revisiting historical data. We build on this correspondence in \ours.

\begin{figure*}[t]
\centering
\includegraphics[width=0.8\linewidth]{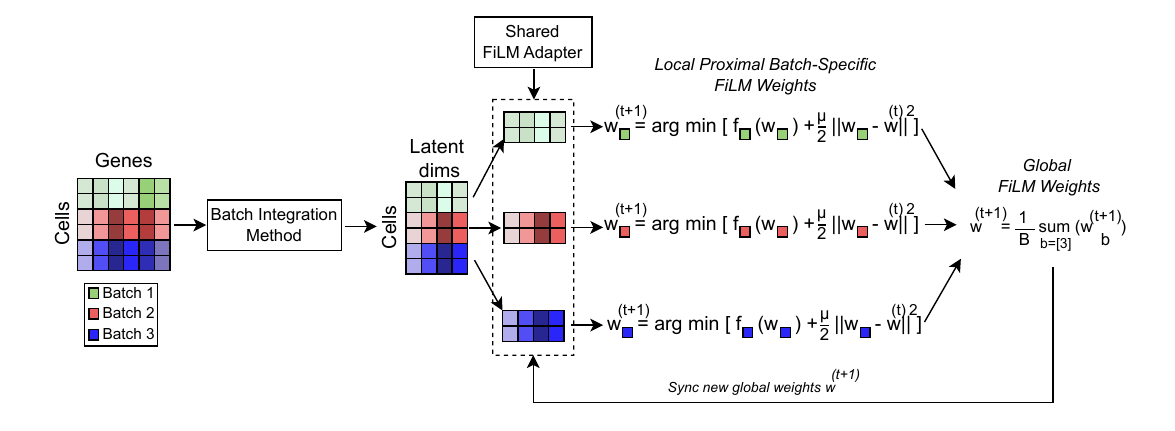}
\caption{Framework of \ours.}
\label{fig:figure 1}
\end{figure*}


\textbf{Model Overview.}
\ours is a lightweight FL-inspired post-hoc framework for refining single-cell latent embeddings while preserving cell-type discriminability under heterogeneous and evolving batch compositions. As shown in Figure~\ref{fig:figure 1}, \ours takes as input a global cell-level embedding matrix produced by an arbitrary embedding construction method, with cells partitioned by batch from the same tissue. Each batch is treated as a local client, as in FL systems. A shared Feature-wise Linear Modulation (FiLM)~\cite{perez2018film} adapter applies batch-conditioned affine updates to the latent embedding. 
We choose FiLM for two reasons that are particularly relevant in evolving single-cell analysis pipelines. First, its scale-and-shift updates provide a conservative form of adaptation, helping preserve the existing embedding geometry as new batches are integrated over time. Second, FiLM introduces only a small number of trainable parameters, making repeated refinement lightweight, scalable, and deployable.
Each client updates its local adapter using only its own embeddings, while a proximal penalty keeps the local update close to the current global adapter. An identity-preserving penalty further discourages unnecessary movement away from the upstream embedding. The updated adapters are aggregated across clients and broadcast for the next round. 

\textbf{Contributions.} We make the following contributions:

\begin{figure}[!t]
\centering
\includegraphics[width=0.65\columnwidth]{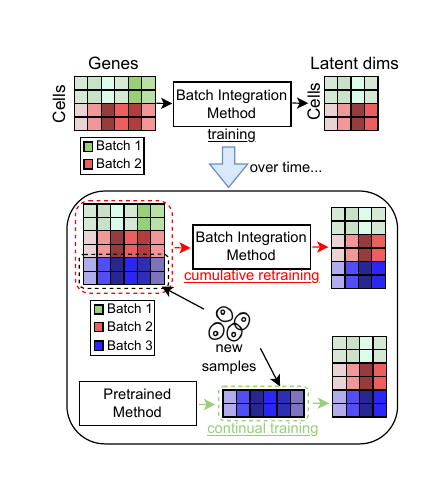}
\caption{
Two dataset evolution scenarios in single-cell analysis. Cumulative retraining recomputes embeddings by retraining the upstream model on all available data, whereas continual training fits newly arriving datasets into a fixed reference embedding without retraining on previously processed data.}
\label{fig:figure 2}
\end{figure}

\begin{itemize}
\item \textbf{Stable cell-type discriminability under heterogeneous cell-type support.} We study post-hoc single-cell embedding refinement in settings where batches contain uneven or incomplete cell-type support. \ours refines precomputed embeddings through batch-conditioned FiLM adapters and conservative regularization. This design aims to preserve cell-type discriminability when specific populations are underrepresented or absent in some batches.

\item \textbf{FL-inspired adaptation for evolving single-cell datasets.} The FL-inspired design of \ours naturally accommodates heterogeneous cell-type support and dataset evolution. We evaluate \ours in both \underline{\emph{cumulative retraining}} settings, where embeddings are periodically recomputed on all available data, and \underline{\emph{continual training}} settings, where new datasets are added to a fixed reference embedding. Figure~\ref{fig:figure 2} illustrates these two dataset evolution scenarios.
  
\item \textbf{Lightweight and deployable post-hoc layer.}
The FiLM parameterization keeps \ours computationally lightweight and deployable on limited computational resources. This makes the refinement step practical for dynamic single-cell analysis pipelines in which new batches arrive over time.
\end{itemize}
\section{Related Work} \label{sec:related_works}

\subsection{Cell-level embedding construction for single-cell integration}

Single-cell integration workflows often begin by transforming high-dimensional gene expression profiles into low-dimensional cell embeddings. Many methods can produce such embeddings. Two representative strategies are batch-agnostic linear dimensionality reduction, exemplified by principal component analysis (PCA), and probabilistic batch-aware modeling, exemplified by single-cell variational inference (scVI)~\cite{lopez2018deep}. PCA provides a simple and efficient representation but does not model batch labels directly. In contrast, scVI learns a latent representation from count data while modeling batch effects within a probabilistic generative framework. These strategies define the upstream embedding space on which downstream correction or refinement is performed. However, end-to-end representation models such as scVI couple embedding construction with batch-aware modeling, so incorporating newly arriving datasets often requires retraining or reference-based adaptation. In this work, our focus is not embedding construction itself, but post-hoc refinement of embeddings produced by arbitrary upstream methods.

\subsection{Post-hoc integration and latent embedding correction}

Harmony~\cite{korsunsky2019fast} and scArches~\cite{lotfollahi2022mapping} are closely related to \ours but address different parts of the integration workflow. Harmony performs post-hoc correction directly on precomputed embeddings. scArches focuses on mapping new query data into pretrained reference models through lightweight adaptation.

\paragraph{Post-hoc correction of precomputed embeddings.}
Harmony is a widely used method for embedding-level batch correction. It improves precomputed embeddings through iterative clustering and dataset-specific correction in latent space, without retraining the upstream representation model. Our setting differs in several important respects. Harmony is typically applied as a global correction procedure over the full embedding matrix. When new batches arrive, the usual workflow reruns correction on the updated collection of cells. This is less aligned with evolving datasets, where previously embedded cells should remain stable as new datasets arrive. In addition, our main focus is not only batch mixing, but also preserving cell-type discriminability when cell-type support varies across batches.

\paragraph{Reference-based adaptation and query mapping.}
scArches extends pretrained single-cell reference atlases through lightweight adaptation of pretrained reference models. While effective for query-to-reference mapping, it relies on compatible pretrained architectures, most commonly scVI and related variational autoencoder models. New data are incorporated through model-specific adaptation rather than direct refinement of existing embeddings. In contrast, \ours operates directly on arbitrary precomputed embeddings and requires neither a pretrained reference model nor architecture-specific adaptation. This makes the refinement procedure applicable across different embedding backbones and naturally suited to settings where datasets evolve over time.

Overall, Harmony and scArches share two related limitations. First, neither is designed for settings where cell-type support varies across batches and new datasets arrive over time. Second, both remain tied to single-cell-specific correction or adaptation frameworks. \ours addresses the first gap through an FL-inspired design that treats batches as heterogeneous local clients. It addresses the second by operating directly on embeddings and batch labels, allowing the same refinement principle to extend beyond single-cell data to heterogeneous embedding distributions.
\section{Methodology}

Given a precomputed cell-level embedding produced by an arbitrary upstream method, \ours performs lightweight post-hoc refinement directly in latent space. The key idea is to perform conservative batch-specific refinement in latent space. To this end, we construct a moment-matched target for each batch and learn a batch-conditioned FiLM adapter whose output approaches this target while remaining close to the identity map. The upstream embedding is fixed throughout training, and optimization is restricted to the adapter parameters.

\subsection{Problem Setup}

Let \(Z \in \mathbb{R}^{N \times d}\) denote a fixed latent embedding matrix for \(N\) cells and latent dimension \(d\). Let \(b_i \in \{1,\ldots,B\}\) be the batch label of cell \(i\), and define
$I_b = \{i: b_i=b\}$, $n_b = |I_b|$, and $Z_b = \{z_i: i\in I_b\}$.
\ours does not use the raw gene expression matrix or cell-type labels during adaptation. All optimization is performed using the fixed embedding \(Z\), the batch labels, and moment targets computed in the same latent space.

For coordinate-wise operations, define the empirical global mean and standard deviation as
\[
\widehat\mu = \frac{1}{N}\sum_{i=1}^N z_i,
\qquad
\widehat s = \left(\frac{1}{N}\sum_{i=1}^N (z_i-\widehat\mu)^{\odot 2}\right)^{1/2},
\]
and the corresponding batch-specific quantities as
\[
\widehat\mu_b = \frac{1}{n_b}\sum_{i\in I_b} z_i,
\qquad
\widehat s_b = \left(\frac{1}{n_b}\sum_{i\in I_b} (z_i-\widehat\mu_b)^{\odot 2}\right)^{1/2}.
\]
Here, powers, square roots, multiplication, and division are applied coordinate-wise. We use a small constant \(\epsilon>0\) for numerical stability and write
$\widehat s^{\epsilon}=\widehat s+\epsilon$ and
$\widehat s_b^{\epsilon}=\widehat s_b+\epsilon$.

\begin{definition}[Latent moment refinement target]
For each batch \(b\), define a coordinate-wise scaling vector
\[
r_b =
\begin{cases}
\widehat s^{\epsilon}\oslash \widehat s_b^{\epsilon}, & \text{if variance matching is enabled},\\
\mathbf{1}, & \text{if only mean matching is used},
\end{cases}
\]
where \(\oslash\) denotes element-wise division. For a cell \(i\in I_b\), the moment-matched target is
\[
t_i = \widehat\mu + r_b \odot (z_i-\widehat\mu_b).
\]
Equivalently, under mean-only matching, \(t_i=z_i+(\widehat\mu-\widehat\mu_b)\). In our default implementation, variance matching is enabled and \(\epsilon=10^{-6}\).
\end{definition}

The target \(t_i\) specifies a batch-level alignment objective in latent space. The learned adapter is encouraged to approach this target while remaining constrained by the adapter regularization.

This construction is compatible with a distributed implementation based on latent summary statistics and each client's local embeddings. Each client only needs to provide \(n_b\), \(\sum_{i\in I_b} z_i\), and \(\sum_{i\in I_b} z_i^{\odot 2}\) for the server to compute global latent moments. The server can then broadcast these moments back to clients for local target construction. Thus, target construction requires neither raw gene expression nor pairwise cell matching across batches.

\begin{proposition}[Latent moment normalization properties]
\label{prop:moment-target}
For every batch \(b\), the target embeddings \(T_b=\{t_i:i\in I_b\}\) have empirical mean
$\frac{1}{n_b}\sum_{i\in I_b} t_i = \widehat\mu$.
Moreover, their coordinate-wise empirical standard deviation is
$\operatorname{Std}(T_b) = r_b \odot \widehat s_b$.
Consequently, under variance matching,
$\operatorname{Std}(T_b)=\widehat s_b \odot (\widehat s^{\epsilon}\oslash \widehat s_b^{\epsilon})$,
and in the idealized case \(\epsilon=0\) with nonzero batch standard deviations, \(\operatorname{Std}(T_b)=\widehat s\) for all batches.
\end{proposition}

\begin{proof}
For any batch \(b\),
\[
\frac{1}{n_b}\sum_{i\in I_b} t_i
=
\widehat\mu
+
r_b\odot
\left(
\frac{1}{n_b}\sum_{i\in I_b}(z_i-\widehat\mu_b)
\right)
=
\widehat\mu.
\]
Since \(t_i-\widehat\mu=r_b\odot(z_i-\widehat\mu_b)\), the coordinate-wise standard deviation of \(T_b\) is \(|r_b|\odot \widehat s_b\). The entries of \(r_b\) are nonnegative by construction, so this equals \(r_b\odot \widehat s_b\). The remaining statements follow by substituting the definition of \(r_b\).
\end{proof}

\begin{proposition}[Within-batch affine preservation]
\label{prop:affine-preservation}
For any two cells \(i,k\in I_b\), the target transformation satisfies
$t_i-t_k = r_b \odot (z_i-z_k)$.
Thus, the target construction aligns batch-level moments while preserving within-batch differences up to a coordinate-wise scaling.
\end{proposition}

\begin{proof}
Subtracting the two target definitions gives
$t_i-t_k = r_b\odot(z_i-\widehat\mu_b)-r_b\odot(z_k-\widehat\mu_b) = r_b\odot(z_i-z_k)$.
\end{proof}

\subsection{FiLM Adapter and Expressivity}

\ours uses a batch-conditioned FiLM adapter. Let \(\Gamma\in\mathbb{R}^{B\times d}\) collect scale vectors \(\gamma_b\), and let \(\Delta\in\mathbb{R}^{B\times d}\) collect shift vectors \(\beta_b\). The adapter parameters are \(\Theta=(\Gamma,\Delta)\), and the refined embedding for cell \(i\) is
$A_{\Theta}(z_i,b_i)=\gamma_{b_i}\odot z_i+\beta_{b_i}$.
The implementation initializes \(\gamma_b=\mathbf{1}\) and \(\beta_b=\mathbf{0}\), so the initial adapter is exactly the identity transformation.

\begin{lemma}[FiLM representability of the moment target]
\label{lem:film-representability}
The moment target belongs to the FiLM adapter class. In particular, for every batch \(b\), define
$\gamma_b^{\star}=r_b$
and
$\beta_b^{\star}=\widehat\mu-r_b\odot\widehat\mu_b$.
Then \(A_{\Theta^{\star}}(z_i,b_i)=t_i\) for all cells \(i\).
\end{lemma}

\begin{proof}
For any \(i\in I_b\),
$A_{\Theta^{\star}}(z_i,b) = r_b\odot z_i + \widehat\mu-r_b\odot\widehat\mu_b = \widehat\mu+r_b\odot(z_i-\widehat\mu_b) = t_i$.
\end{proof}

This lemma gives the basic mathematical reason why \ours can perform latent refinement through lightweight batch-conditioned affine adaptation: the target moment correction is itself a diagonal affine map over latent coordinates, and the FiLM adapter parameterizes exactly this family of maps.

\subsection{Latent Convex Refinement View}

The above construction gives a useful way to interpret \ours. The upstream representation model has already transformed the original high-dimensional expression profile into a fixed latent feature vector. Once these features are fixed, \ours no longer optimizes a nonlinear representation model. It optimizes a small affine adapter on top of the frozen latent representation.

For each coordinate \(j\) and each cell \(i\in I_b\), define
$\theta_{b j}=\begin{bmatrix}\gamma_{b j} \\ \beta_{b j}\end{bmatrix}$ 
and
$x_{i j}=\begin{bmatrix}z_{i j} \\ 1\end{bmatrix}$.

Then the \(j\)-th coordinate of the adapter output can be written as
$A_{\Theta}(z_i,b)_j=x_{i j}^{\top}\theta_{b j}$.
Thus, after the embedding \(Z\) and targets \(T\) are fixed, fitting the FiLM adapter reduces to a regularized linear least-squares optimization problem over frozen latent features.

A centralized counterpart of the adapter objective can be written as the global latent refinement objective
\[
\begin{aligned}
\mathcal{F}(\Theta)
&=
\sum_{b=1}^B \frac{n_b}{N}
\left[
\frac{\lambda_{\mathrm{target}}}{n_b d}
\sum_{i\in I_b}\left\|A_{\Theta}(z_i,b)-t_i\right\|_2^2
\right] \\
&\quad+
\frac{\lambda_{\mathrm{id}}}{Bd}
\left(
\|\Gamma-\mathbf{1}\|_F^2+
\|\Delta\|_F^2
\right).
\end{aligned}
\]
The weights \(n_b/N\) make the data-fitting terms a sample-size-weighted average of client objectives, matching the aggregation weights used during federated optimization. The identity regularizer matches the implementation, where deviations from \(\gamma_b=\mathbf{1}\) and \(\beta_b=\mathbf{0}\) are penalized by the mean squared parameter deviation across all batch-indexed adapter entries.

\begin{proposition}[Convex latent refinement objective]
\label{prop:global-convex-objective}
The global objective \(\mathcal{F}\) is a convex quadratic function of the adapter parameters \(\Theta=(\Gamma,\Delta)\). If \(\lambda_{\mathrm{id}}>0\), then \(\mathcal{F}\) is strongly convex and has a unique minimizer.
\end{proposition}

\begin{proof}
For fixed \(Z\) and \(T\), each scalar prediction has the linear form
\(A_{\Theta}(z_i,b)_j=x_{ij}^{\top}\theta_{bj}\). Therefore, each target-fitting term is a squared affine function of the adapter parameters, hence a convex quadratic. The identity regularizer is also a convex quadratic and contributes positive curvature to every scale and shift coordinate when \(\lambda_{\mathrm{id}}>0\). Therefore, the full objective is convex, and it is strongly convex under \(\lambda_{\mathrm{id}}>0\).
\end{proof}

This convexification applies only to the adapter stage. The upstream representation model remains unchanged throughout refinement. \ours uses the upstream method to obtain fixed latent features and then solves a well-posed convex correction problem in that latent space. This distinction is important in heterogeneous federated settings: heterogeneous local updates are regularized in a low-dimensional adapter space rather than in the full nonconvex representation model.

\subsection{Local Objective and Conservative Regularization}

At communication round \(t\), each client receives the current global adapter \(\Theta^{(t)}=(\Gamma^{(t)},\Delta^{(t)})\) and optimizes a local copy using its own data. For client \(b\), the full-batch empirical local objective is
\[
\begin{aligned}
\mathcal{L}_b(\Theta;\Theta^{(t)})
&=
\frac{\lambda_{\mathrm{target}}}{n_b d}
\sum_{i\in I_b}
\left\|A_{\Theta}(z_i,b)-t_i\right\|_2^2 \\
&\quad+
\mu\left\|\Theta-\Theta^{(t)}\right\|_F^2 \\
&\quad+
\frac{\lambda_{\mathrm{id}}}{Bd}
\left(
\|\Gamma-\mathbf{1}\|_F^2+
\|\Delta\|_F^2
\right),
\end{aligned}
\]
where
\[
\left\|\Theta-\Theta^{(t)}\right\|_F^2
=
\left\|\Gamma-\Gamma^{(t)}\right\|_F^2+
\left\|\Delta-\Delta^{(t)}\right\|_F^2.
\]
The first term fits the adapter output to the moment-matched target. The FedProx~\cite{li2020federated} term keeps the local adapter close to the current global adapter. The identity term penalizes deviations from \(\gamma_b=\mathbf{1}\) and \(\beta_b=\mathbf{0}\). In implementation, each client carries a full local copy of the adapter state. Only the data-fitting term is batch-specific. The proximal and identity terms regularize the full local copy before aggregation. Importantly, the moment correction is not enforced exactly in practice. 
The proximal and identity-preserving terms bias the adapter toward conservative latent changes, discouraging large shifts under heterogeneous batch compositions.

\begin{proposition}[Well-posed local latent optimization]
\label{prop:local-wellposed}
If \(\mu+\lambda_{\mathrm{id}}/(Bd)>0\), then \(\mathcal{L}_b(\Theta;\Theta^{(t)})\) is a strongly convex quadratic function of the adapter parameters \(\Theta\). Therefore, each local subproblem has a unique minimizer.
\end{proposition}

\begin{proof}
The map \(A_{\Theta}(z_i,b)\) is affine in \(\Theta\). Hence, the target-fitting term is a convex quadratic function of \(\Theta\). The proximal and identity terms are positive quadratic penalties around fixed parameter values. Their combined Hessian contributes a positive multiple of the identity whenever \(\mu+\lambda_{\mathrm{id}}/(Bd)>0\). Therefore, the full objective is strongly convex and has a unique minimizer.
\end{proof}

The optimizer used in practice is Adam applied to mini-batches from each client loader. On each mini-batch, the target-fitting term is computed as a mean squared error over cells and latent dimensions, while the FedProx and identity terms are applied to the full local adapter state. Thus, the implementation performs stochastic iterative optimization of the above well-defined latent quadratic objective rather than updating the upstream embedding model.

\subsection{Federated Optimization Procedure}

Training proceeds in communication rounds. At the beginning of each round, the current global adapter \(\Theta^{(t)}\) is copied to every client. Each client performs local optimization on triples \((z_i,b_i,t_i)\) from its local data loader and returns the full updated adapter state, denoted \(\Theta_b^{(t+1)}\). The server then applies sample-size-weighted federated averaging:
\[
\Theta^{(t+1)}
=
\frac{1}{\sum_{b=1}^B n_b}
\sum_{b=1}^B n_b\Theta_b^{(t+1)}.
\]
This procedure follows the implementation in which each client loader uses the full local batch dataset with seeded shuffling, and client updates are weighted by the number of local cells during aggregation.

\begin{theorem}[Conservative latent refinement by \ours]
\label{thm:latent-moment-refinement}
For fixed embeddings \(Z\) and moment targets \(T\), \ours defines a latent-space refinement procedure over FiLM adapter parameters. The FiLM adapter can represent the moment-matched target exactly. Under this target, each batch is aligned to the global latent mean. Under ideal variance matching with \(\epsilon=0\) and nonzero batch standard deviations, each batch also matches the global coordinate-wise standard deviation. The resulting fixed-feature objective is a convex quadratic in the adapter parameters. The proximal and identity-preserving penalties define strongly convex local objectives that bias client updates toward conservative latent changes before sample-size-weighted aggregation.
\end{theorem}
\begin{proof}
The existence of an exact moment-target adapter follows from Lemma~\ref{lem:film-representability}, or directly from the construction \(\gamma_b^{\star}=r_b\) and \(\beta_b^{\star}=\widehat\mu-r_b\odot\widehat\mu_b\). The aligned mean and standard deviation properties follow from Proposition~\ref{prop:moment-target}. The fixed-feature global objective is convex quadratic by Proposition~\ref{prop:global-convex-objective}. The local optimization objective is strongly convex by Proposition~\ref{prop:local-wellposed}. Therefore, \ours implements a conservative latent refinement procedure by optimizing a regularized moment-alignment objective in the adapter space.
\end{proof}


\subsection{Inference and Post-hoc Refinement}

After training, the learned adapter is applied once to the full embedding matrix. Given \(Z\) and batch labels \(\{b_i\}_{i=1}^N\), the final refined embedding is
$\widetilde Z = \left\{ \gamma_{b_i}\odot z_i+\beta_{b_i} \right\}_{i=1}^N$.
The resulting embedding can be used directly in standard single-cell workflows, including visualization, cell-type classification, and integration benchmarking. Since \ours only operates on precomputed latent embeddings and optimizes only FiLM adapter parameters, it remains compatible with arbitrary upstream representation methods.
\section{Experiments}

We evaluate whether \ours stabilizes precomputed single-cell embeddings under heterogeneous and evolving batch compositions. The experiments address three questions. First, does \ours improve downstream cell-type classification on imbalanced multi-batch and cross-study datasets? Second, does it preserve affected-cell-type discriminability when a population is downsampled or removed from one batch? Third, does it remain effective when datasets arrive over time in cumulative retraining and continual integration settings?

All experiments apply \ours as a lightweight post-hoc stage on top of upstream embeddings. On each batch, local training uses Adam~\cite{kingma2015adam} with learning rate \(5\times 10^{-2}\) for three local epochs per federated round. Each local objective includes a target-matching loss toward the precomputed latent refinement targets and a FedProx proximal penalty with \(\mu=10^{-3}\). We set \(\lambda_{\mathrm{target}}=0.5\), \(\lambda_{\mathrm{id}}=10^{-3}\), enable variance matching by default, and use \(\epsilon=10^{-6}\) for numerical stability. We use a batch size of 256, and each client loader uses the full local batch dataset with deterministic shuffling.

\subsection{Benchmark}

\subsubsection{Multi-batch and Cross-study Datasets with Heterogeneous Cell-Type Support.}
We evaluate \ours on datasets that contain both batch effects and heterogeneous cell-type support. The first setting captures within-study technical variation across experimental batches, such as repeated measurements or processing runs within a study. We refer to this setting as \emph{multi-batch heterogeneity}. The second setting captures between-study variation from independently designed studies that may differ in protocol, platform, or laboratory. We refer to this setting as \emph{cross-study heterogeneity}. Both settings become more difficult when cell-type support varies across batches. For multi-batch heterogeneity, we use real scRNA-seq datasets from~\cite{maan2024characterizing}: the imbalanced four-batch PBMC dataset~\cite{ding2020systematic} and the six-batch mouse hindbrain development (MHD) dataset~\cite{vladoiu2019childhood}. To evaluate cross-study heterogeneity, we use a two-batch Human Pancreas Multi-Study (HPMS) dataset.

For all datasets, we select 2,500 highly variable genes and normalize the data using Scanpy~\cite{wolf2018scanpy} by scaling each cell's read counts to 10,000 followed by a log1p transformation. For HPMS, we adopt the preprocessed version~\cite{scvi_surgery_pipeline} used by the scVI$-$scArches pipeline, where gene filtering and highly variable gene selection have already been performed. Accordingly, we apply only normalization and log1p transformation. Dataset summary statistics are reported in Table~\ref{tab:dataset_statistics} of Appendix~\ref{app:data_stats}.

\subsubsection{Metrics.}
We use macro-F1 to evaluate downstream cell-type classification under varying cell-type support. We assess batch integration performance using the scIB benchmarking framework~\cite{luecken2022benchmarking}, including biological conservation metrics (NMI, ARI, ASW, isolated-label F1, cLISI) and batch-removal metrics (batch ASW, iLISI, kBET, graph connectivity, and PCR). Full metric definitions and aggregation procedures are provided in Appendix~\ref{app:complete_scib}.

\subsubsection{Baselines.}
We compare \ours with Harmony and scArches for cell-type classification. We also evaluate whether \ours can refine PCA and scVI embeddings under cumulative retraining and continual integration settings. For PCA, we fix the latent dimensionality to 40. For scVI, Harmony, and scArches, we use the default settings from their respective implementations.

\subsection{Results}
\subsubsection{Cell-Type Classification under Heterogeneous Batch Composition.}

\begin{table*}[!t]
\centering
\small
\setlength{\tabcolsep}{3pt}
\resizebox{\textwidth}{!}{
\begin{tabular}{c|c|cc|cc|cc}
\hline
\makecell[c]{\textbf{Upstream}\\\textbf{Method}} & \textbf{Model}
& \multicolumn{2}{c|}{\textbf{4-batch PBMC}}
& \multicolumn{2}{c|}{\textbf{6-batch MHD}}
& \multicolumn{2}{c}{\textbf{2-batch HPMS}} \\
\cline{3-8}
& & \textbf{Macro-F1} & \textbf{Runtime (s)}
  & \textbf{Macro-F1} & \textbf{Runtime (s)}
  & \textbf{Macro-F1} & \textbf{Runtime (s)} \\
\hline

\multirow{2}{*}{PCA}
& Harmony
& $0.6633_{\pm 0.0044}$ & $14.2332_{\pm 1.7650}$
& $0.3829_{\pm 0.0023}$ & $3.2045_{\pm 0.3022}$
& $0.8706_{\pm 0.0007}$ & $21.2656_{\pm 1.7654}$ \\

& \cellcolor{gray!12}\ours
& \cellcolor{gray!12}$\mathbf{0.7594}_{\pm 0.0008}$ & \cellcolor{gray!12}$\mathbf{5.3492}_{\pm 0.9618}$
& \cellcolor{gray!12}$\mathbf{0.6567}_{\pm 0.001}$ & \cellcolor{gray!12}$\mathbf{1.98}_{\pm 0.2187}$
& \cellcolor{gray!12}$\mathbf{0.8710}_{\pm 0.0010}$ & \cellcolor{gray!12}$\mathbf{5.4333}_{\pm 0.5713}$ \\

\hline\hline

\multirow{3}{*}{scVI}
& scArches
& $0.4415_{\pm 0.0000}$ & $16.6963_{\pm 2.7281}$
& $0.1643_{\pm 0.0000}$ & $14.8221_{\pm 0.2786}$
& $0.7745_{\pm 0.0000}$ & $38.0826_{\pm 0.5940}$ \\

& Harmony
& $0.7010_{\pm 0.0071}$ & $7.1867_{\pm 0.9681}$
& $0.2683_{\pm 0.0051}$ & $\mathbf{5.9330}_{\pm 4.3050}$
& $0.8502_{\pm 0.002}$ & $12.0450_{\pm 1.7893}$ \\

& \cellcolor{gray!12}\ours
& \cellcolor{gray!12}$\mathbf{0.7504}_{\pm 0.0001}$ & \cellcolor{gray!12}$\mathbf{0.5540}_{\pm 0.1382}$
& \cellcolor{gray!12}$\mathbf{0.4747}_{\pm 0.0002}$ & \cellcolor{gray!12}$18.4768_{\pm 0.6348}$
& \cellcolor{gray!12}$\mathbf{0.8613}_{\pm 0.0002}$ & \cellcolor{gray!12}$\mathbf{0.8574}_{\pm 0.1798}$ \\

\hline
\end{tabular}
}
\caption{
Measurement of cell-type classification performance using macro-F1 score (the higher, the better).
All results are reported as mean $\pm$ standard deviation over 5 independent runs.
\ours improves downstream classification performance across datasets and embedding backbones, with competitive runtime in most settings.
}
\label{tab:f1_macro}
\end{table*}

We train a balanced logistic regression classifier on the refined embeddings to predict cell-type labels. A fixed stratified train/test split (80\% / 20\%) is used across all models to ensure fair comparison. 

Table~\ref{tab:f1_macro} shows that \ours improves macro-F1 across all datasets and upstream embeddings. The largest gains occur on MHD, where cell-type support differs most strongly across batches. This suggests that conservative refinement becomes increasingly useful as composition heterogeneity increases. Notably, the improvement persists on both PCA and scVI embeddings, indicating that the refinement effect is not tied to a specific upstream representation.


\subsubsection{Stability under Imbalanced Batch Compositions.}
\begin{figure}[!t]
\centering
\includegraphics[width=0.7\linewidth]{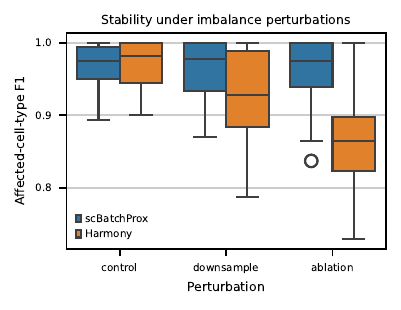}
\caption{
Stability under imbalance perturbations measured using affected-cell-type F1 on post-hoc refinement of PCA embeddings. \ours maintains substantially more stable affected-cell-type F1 across perturbation regimes.
}
\label{fig:figure 5}
\end{figure}

To evaluate stability under reduced cell-type support, we use a balanced two-batch PBMC dataset derived from~\cite{ding2020systematic} and used in~\cite{maan2024characterizing}, containing six shared cell populations with 200 cells per cell type in each batch. We perturb one selected cell type in one batch using two regimes: \textbf{downsample}, where only 10\% of cells are retained, and \textbf{ablation}, where the population is removed entirely. We report affected-cell-type F1, defined as the F1 score computed only on the perturbed cell type after post-hoc refinement of PCA embeddings.

As shown in Figure~\ref{fig:figure 5}, under balanced conditions, both Harmony and \ours achieve similarly strong affected-cell-type F1 scores. However, as cell-type support becomes reduced or entirely absent within one batch, Harmony exhibits substantially larger degradation, particularly in the ablation setting. In contrast, \ours maintains substantially more stable performance across perturbation regimes, suggesting that conservative latent refinement better preserves affected-cell-type discriminability in this perturbation setting.

\subsubsection{Dataset Evolution.}

We next ask whether \ours remains effective as datasets evolve over time. We first evaluate continual training, where newly arriving datasets are incorporated without retraining on previously processed data. The initial protocols define a reference embedding, and later datasets are added incrementally. Figure~\ref{fig:figure 4} shows that \ours consistently improves aggregate scIB scores across both protocol arrival orders, indicating that conservative latent refinement remains effective under continual dataset growth.

\begin{figure}[!t]
\centering
\includegraphics[width=0.9\columnwidth]{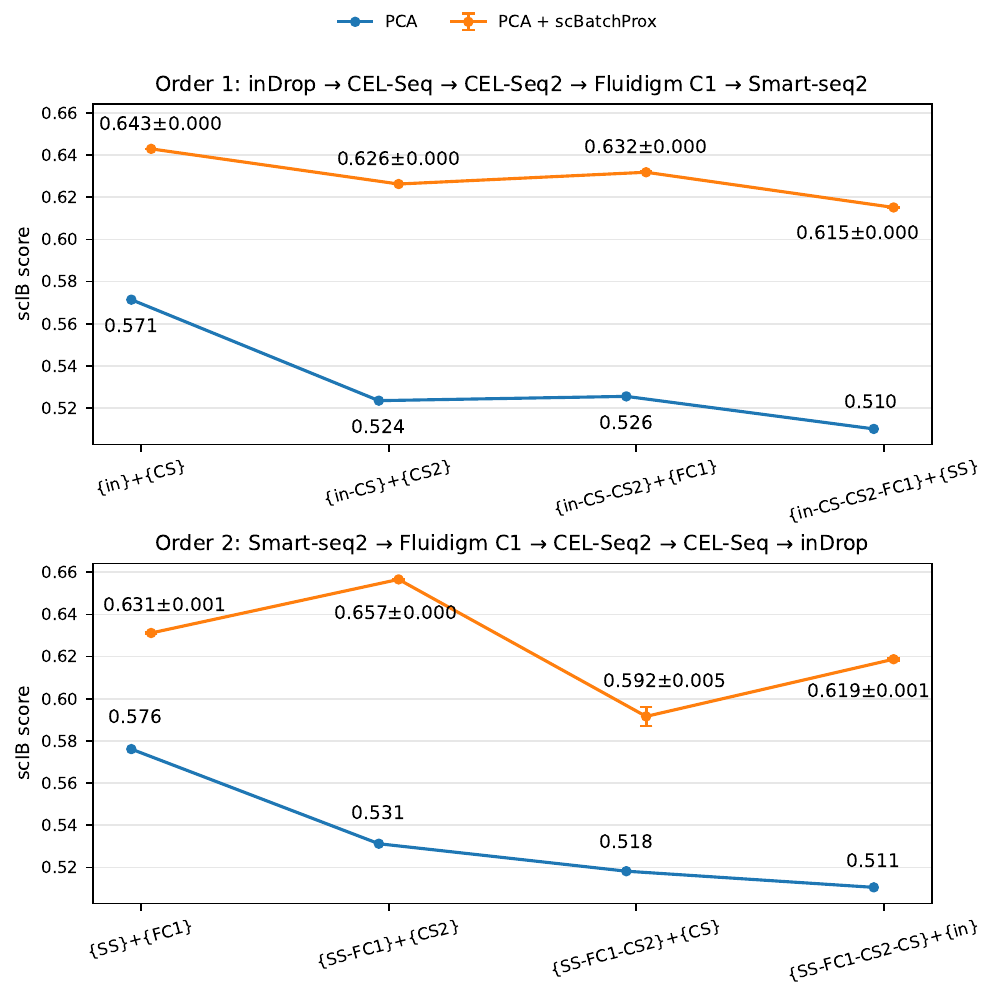}
\caption{
Aggregate scIB scores for continual training under protocol arrival:
(a) technological progression order
(inDrop $\to$ CEL-Seq $\to$ CEL-Seq2 $\to$ Fluidigm C1 $\to$ Smart-seq2)
and
(b) reverse order
(Smart-seq2 $\to$ Fluidigm C1 $\to$ CEL-Seq2 $\to$ CEL-Seq $\to$ inDrop).
Continual training is commonly encountered in practice, yet most end-to-end 
batch-aware models do not support it. Despite this constraint, \ours consistently 
improves embedding quality across both arrival orders.
}
\label{fig:figure 4}
\end{figure}

We additionally evaluate cumulative retraining, where all available datasets are re-embedded whenever new data arrive. Similar improvements are observed across both protocol arrival orders (Figure~\ref{fig:figure 3} of Appendix~\ref{app:wall_clock}). Runtime comparisons in Table~\ref{tab:wall-clock-table} of Appendix~\ref{app:wall_clock} further show that cumulative runtime grows substantially faster for workflows that repeatedly recompute upstream embeddings, whereas \ours adds only a lightweight post-hoc refinement stage.

\subsubsection{Beyond Single-Cell Embeddings.}

\begin{table}[!t]
\centering
\small
\setlength{\tabcolsep}{6pt}
\resizebox{\columnwidth}{!}{
\begin{tabular}{c|c|cc}
\hline
\makecell[c]{\textbf{Upstream}\\\textbf{Method}} & \textbf{Model}
& \multicolumn{2}{c}{\textbf{QM8}} \\
\cline{3-4}
& & \textbf{Macro-F1} & \textbf{Runtime (s)} \\
\hline

\multirow[c]{2}{*}{Morgan}
& Harmony
& $0.3065_{\pm 0.0108}$ & $22.6302_{\pm 6.1664}$ \\

& \cellcolor{gray!12} \ours
& \cellcolor{gray!12}$\mathbf{0.6996}_{\pm 0.0000}$
& \cellcolor{gray!12}$\mathbf{8.29}_{\pm 2.26}$ \\

\hline
\end{tabular}
}
\caption{Representation quality on QM8 measured by macro-F1 on classification over structurally defined groups.}
\label{tab:quantum}
\end{table}

We demonstrate the broader applicability of \ours beyond single-cell data using the QM8 dataset~\cite{qm8data} from MoleculeNet~\cite{wu2018moleculenetbenchmarkmolecularmachine}. This dataset consists of small molecules with quantum mechanical targets. Molecules are converted from SMILES strings into fixed-length feature vectors using Morgan fingerprints~\cite{morgan1965fingerprint}. To simulate heterogeneous and imbalanced latent distributions, we partition samples based on heavy atom count, resulting in structurally distinct groups with substantially different sample support (see Appendix~\ref{app:stat_qm8}, Table~\ref{tab:qm8_batches}). These fingerprint representations are treated as embeddings for post-hoc refinement.

We evaluate discriminability across these structurally defined groups using macro-F1. As shown in Table~\ref{tab:quantum}, \ours substantially outperforms Harmony on this task. Notably, the setting contains neither gene-expression measurements nor cell-type annotations, indicating that the benefits of conservative refinement are not tied to single-cell-specific assumptions. This result supports the broader applicability of \ours to heterogeneous embedding distributions beyond single-cell data.

\section{Ablation study}

\begin{table}[t]
\caption{
Ablation study measured using macro-F1 score (higher is better). 
All results are reported as mean $\pm$ standard deviation over 5 independent runs.
}
\label{tab:ablation_macrof1}

\centering
\small
\setlength{\tabcolsep}{4pt}

\begin{tabular}{lccc}
\hline
\textbf{Variant}
& \textbf{4-batch PBMC}
& \textbf{6-batch MHD}
& \textbf{2-batch HPMS} \\
\hline

Full model
& $\mathbf{0.7585}_{\pm 0.0009}$
& $0.6140_{\pm 0.0006}$
& $0.8725_{\pm 0.0005}$ \\

mean-only
& $0.7478_{\pm 0.0006}$
& $0.6173_{\pm 0.0004}$
& $0.8718_{\pm 0.0006}$ \\

w/o FedProx
& $0.7471_{\pm 0.0000}$
& $\mathbf{0.6571}_{\pm 0.0007}$
& $\mathbf{0.8741}_{\pm 0.0011}$ \\

w/o identity
& $0.7577_{\pm 0.0000}$
& $0.6121_{\pm 0.0011}$
& $0.8716_{\pm 0.0013}$ \\

\hline
\end{tabular}
\end{table}

The ablation results show that the contribution of each regularization component depends on dataset heterogeneity. Identity-preserving regularization is consistently beneficial, whereas the effect of FedProx varies across datasets. In particular, removing FedProx improves performance on MHD and HPMS, suggesting that highly heterogeneous datasets may benefit from more flexible local adaptation. Mean-only refinement generally underperforms the full model, supporting the use of variance-aware moment refinement.



\section{Conclusion}

We introduced \ours, a lightweight post-hoc latent refinement framework for stabilizing single-cell embeddings under heterogeneous and evolving batch compositions. By operating directly on precomputed embeddings, \ours improves downstream cell-type classification and maintains more stable affected-cell-type discriminability under reduced cell-type support. The method remains effective in both cumulative retraining and continual integration settings while requiring only lightweight adapter optimization. 
Future work will investigate when post-hoc refinement becomes preferable to repeated global retraining and whether atlas-scale integration can be decomposed into many smaller refinement tasks under limited computational resources.

\section*{Acknowledgments}
We acknowledge the Auburn University Easley Cluster for support of this work.

\bibliography{ref}

@article{macosko2015highly,
  title={{Highly parallel genome-wide expression profiling of individual cells using nanoliter droplets}},
  author={Macosko, Evan Z. and Basu, Anindita and Satija, Rahul and Nemesh, James and Shekhar, Karthik and Goldman, Melissa and Tirosh, Itay and Bialas, Allison R and Kamitaki, Nolan and Martersteck, Emily M. and  Trombetta, John J. and Weitz, David A. and Sanes, Joshua R. and Shalek, Alex K. and Regev, Aviv and McCarroll, Steven A.},
  journal={Cell},
  volume={161},
  number={5},
  pages={1202--1214},
  year={2015},
  publisher={Elsevier}
}

@article{klein2015droplet,
  title={{Droplet barcoding for single-cell transcriptomics applied to embryonic stem cells}},
  author={Klein, Allon M. and Mazutis, Linas and Akartuna, Ilke and Tallapragada, Naren and Veres, Adrian and Li, Victor and Peshkin, Leonid and Weitz, David A. and Kirschner, Marc W.},
  journal={Cell},
  volume={161},
  number={5},
  pages={1187--1201},
  year={2015},
  publisher={Elsevier}
}

@article{satpathy2019massively,
  title={{Massively parallel single-cell chromatin landscapes of human immune cell development and intratumoral T cell exhaustion}},
  author={Satpathy, Ansuman T. and Granja, Jeffrey M. and Yost, Kathryn E. and Qi, Yanyan and Meschi, Francesca and McDermott, Geoffrey P. and Olsen, Brett N. and Mumbach, Maxwell R. and Pierce, Sarah E. and Corces, M. Ryan and  Shah, Preyas and Bell, Jason C. and Jhutty, Darisha and Nemec, Corey M. and Wang, Jean and Wang, Li and Yin, Yifeng and Giresi, Paul G. and Chang, Anne Lynn S. and Zheng, Grace X. Y. and Greenleaf, William J. and Chang, Howard Y.},
  journal={Nature Biotechnology},
  volume={37},
  number={},
  pages={925--936},
  year={2019},
  publisher={Nature Publishing Group US New York}
}

@article{datlinger2021ultra,
  title={{Ultra-high-throughput single-cell RNA sequencing and perturbation screening with combinatorial fluidic indexing}},
  author={Datlinger, Paul and Rendeiro, Andr{\'e} F and Boenke, Thorina and Senekowitsch, Martin and Krausgruber, Thomas and Barreca, Daniele and Bock, Christoph},
  journal={Nature Methods},
  volume={18},
  number={6},
  pages={635--642},
  year={2021},
  publisher={Nature Publishing Group US New York}
}

@article{lopez2018deep,
  title={{Deep generative modeling for single-cell transcriptomics}},
  author={Lopez, Romain and Regier, Jeffrey and Cole, Michael B. and Jordan, Michael I. and Yosef, Nir},
  journal={Nature Methods},
  volume={15},
  number={},
  pages={1053--1058},
  year={2018},
  publisher={Nature Publishing Group US New York}
}

@inproceedings{li2020federated,
  title={{Federated optimization in heterogeneous networks}},
  author={Li, Tian and Sahu, Anit Kumar and Zaheer, Manzil and Sanjabi, Maziar and Talwalkar, Ameet and Smith, Virginia},
  booktitle={Proceedings of Machine Learning and Systems},
  volume={2},
  pages={429--450},
  year={2020},
  publisher={mlsys.org},
  address={Austin, TX, USA}
}

@article{luecken2022benchmarking,
  title={{Benchmarking atlas-level data integration in single-cell genomics}},
  author={Luecken, Malte D and B{\"u}ttner, Maren and Chaichoompu, Kridsadakorn and Danese, Anna and Interlandi, Marta and M{\"u}ller, Michaela F and Strobl, Daniel C and Zappia, Luke and Dugas, Martin and Colom{\'e}-Tatch{\'e}, Maria and  Theis, Fabian J.},
  journal={Nature Methods},
  volume={19},
  number={},
  pages={41--50},
  year={2022},
  publisher={Nature Publishing Group US New York}
}

@article{wolf2018scanpy,
  title={{SCANPY: large-scale single-cell gene expression data analysis}},
  author={Wolf, F. Alexander and Angerer, Philipp and Theis, Fabian J.},
  journal={Genome Biology},
  volume={19},
  number={},
  pages={15},
  year={2018},
  publisher={Springer}
}

@inproceedings{perez2018film,
  title={{Film: Visual reasoning with a general conditioning layer}},
  author={Perez, Ethan and Strub, Florian and De Vries, Harm and Dumoulin, Vincent and Courville, Aaron},
  booktitle={Proceedings of the Thirty-Second AAAI Conference on Artificial Intelligence and Thirtieth Innovative Applications of Artificial Intelligence Conference and Eighth AAAI Symposium on Educational Advances in Artificial Intelligence},
  volume={32},
  pages={3942--3951},
  year={2018},
  publisher={AAAI Press},
  address={New Orleans, Louisiana, USA}
}

@inproceedings{kingma2015adam,
  title={{Adam: A Method for Stochastic Optimization}},
  author={Kingma, Diederik P. and Ba, Jimmy},
  booktitle={3rd International Conference on Learning Representations, ICLR 2015},
  volume={},
  pages={},
  year={2015},
  publisher={},
  address={San Diego, CA, USA}
}

@article{morgan1965fingerprint,
  title={{The Generation of a Unique Machine Description for Chemical Structures-A Technique Developed at Chemical Abstracts Service}},
  author={Morgan, Harry L.},
  journal={Journal of Chemical Documentation},
  volume={5},
  number={2},
  pages={107--113},
  year={1965},
  publisher={ACS Publications}
}

@misc{wu2018moleculenetbenchmarkmolecularmachine,
      title={{MoleculeNet: A Benchmark for Molecular Machine Learning}}, 
      author={Zhenqin Wu and Bharath Ramsundar and Evan N. Feinberg and Joseph Gomes and Caleb Geniesse and Aneesh S. Pappu and Karl Leswing and Vijay Pande},
      eprint={1703.00564},
      archivePrefix={arXiv},
      year={2018},
      url={https://arxiv.org/abs/1703.00564}, 
}

@article{korsunsky2019fast,
  title={{Fast, sensitive and accurate integration of single-cell data with Harmony}},
  author={Korsunsky, Ilya and Millard, Nghia and Fan, Jean and Slowikowski, Kamil and Zhang, Fan and Wei, Kevin and Baglaenko, Yuriy and Brenner, Michael and Loh, Po-ru and Raychaudhuri, Soumya},
  journal={Nature Methods},
  volume={16},
  number={},
  pages={1289--1296},
  year={2019},
  publisher={Nature Publishing Group US New York}
}

@article{lotfollahi2022mapping,
  title={{Mapping single-cell data to reference atlases by transfer learning}},
  author={Lotfollahi, Mohammad and Naghipourfar, Mohsen and Luecken, Malte D. and Khajavi, Matin and B{\"u}ttner, Maren and Wagenstetter, Marco and Avsec, {\v Z}iga and Gayoso, Adam and Yosef, Nir and Interlandi, Marta and Rybakov, Sergei and Misharin, Alexander V. and Theis, Fabian J.},
  journal={Nature Biotechnology},
  volume={40},
  number={},
  pages={121--130},
  year={2022},
  publisher={Nature Publishing Group US New York}
}

@article{maan2024characterizing,
  title={{Characterizing the impacts of dataset imbalance on single-cell data integration}},
  author={Maan, Hassaan and Zhang, Lin and Yu, Chengxin and Geuenich, Michael J. and Campbell, Kieran R. and Wang, Bo},
  journal={Nature Biotechnology},
  volume={42},
  pages={1899--1908},
  year={2024},
  publisher={Nature Publishing Group US New York}
}

@article{vladoiu2019childhood,
  title={{Childhood cerebellar tumours mirror conserved fetal transcriptional programs}},
  author={Vladoiu, Maria C. and El-Hamamy, Ibrahim and Donovan, Laura K. and Farooq, Hamza and Holgado, Borja L. and Sundaravadanam, Yogi and Ramaswamy, Vijay and Hendrikse, Liam D. and Kumar, Sachin and Mack, Stephen C. and Lee, John J. Y. and Fong, Vernon and Juraschka, Kyle and Przelicki, David and Michealraj, Antony and Skowron, Patryk and Luu, Betty and Suzuki, Hiromichi and Morrissy, A. Sorana and Cavalli, Florence M. G. and Garzia, Livia and Daniels, Craig and Wu, Xiaochong and Qazi, Maleeha A. and Singh, Sheila K. and Chan, Jennifer A. and Marra, Marco A. and Malkin, David and Dirks, Peter and Heisler, Lawrence and Pugh, Trevor and Ng, Karen and Notta, Faiyaz and Thompson, Eric M. and Kleinman, Claudia L. and Joyner, Alexandra L. and Jabado, Nada and Stein, Lincoln and Taylor, Michael D.},
  journal={Nature},
  volume={572},
  number={7767},
  pages={67--73},
  year={2019},
  publisher={Nature Publishing Group UK London}
}

@article{ding2020systematic,
  title={{Systematic comparison of single-cell and single-nucleus RNA-sequencing methods}},
  author={Ding, Jiarui and Adiconis, Xian and Simmons, Sean K. and Kowalczyk, Monika S. and Hession, Cynthia C. and Marjanovic, Nemanja D. and Hughes, Travis K. and Wadsworth, Marc H. and Burks, Tyler and Nguyen, Lan T. and Kwon, John Y. H. and Barak, Boaz and Ge, William and Kedaigle, Amanda J. and Carroll, Shaina and Li, Shuqiang and Hacohen, Nir and Rozenblatt-Rosen, Orit and Shalek, Alex K. and Villani, Alexandra-Chloé and Regev, Aviv and Levin, Joshua Z.},
  journal={Nature Biotechnology},
  volume={38},
  number={6},
  pages={737--746},
  year={2020},
  publisher={Nature Publishing Group US New York}
}

@misc(scvi_surgery_pipeline,
    author = {Lotfollahi, Mohammad and Rybakov, Sergei and Naghipourfar, Mohsen},
    title = {{Unsupervised surgery pipeline with SCVI}},
    howpublished="\url{https://docs.scarches.org/en/latest/scvi_surgery_pipeline.html}",
    note="Accessed: 2026-05-29",
    year = {2025}
  )

@misc(qm8data,
    author = {Wu, Zhenqin and Ramsundar, Bharath and Feinberg, Evan N. and Gomes, Joseph and Geniesse, Caleb and Pappu, Aneesh S. and Leswing, Karl},
    title = {{Dataset Collection}},
    howpublished="\url{https://moleculenet.org/datasets-1}",
    note="Accessed: 2026-05-29",
    year = {2026}
  )


\appendix

\section{Data Statistics} \label{app:data_stats}

Table~\ref{tab:dataset_statistics} summarizes the statistics of all datasets used in this work, including the batch-wise cell-type compositions to illustrate the degree of compositional imbalance across batches. Among the three datasets, 6-batch MHD exhibits the most severe heterogeneity in cell-type distributions, with several cell types appearing exclusively or predominantly in specific batches.

\begin{table*}[!htbp]
\centering
\small
\setlength{\tabcolsep}{4pt}
\renewcommand{\arraystretch}{1.15}
\begin{tabularx}{\textwidth}{p{2.2cm} p{3.2cm} p{1.4cm} X}
\hline
Dataset & Batch & \# Cells & \# Cells per Cell Type \\
\hline

\multirow{4}{=}{4-batch PBMC}
& 10x & 3,222 &
Cytotoxic T: 962; CD4+ T: 960; B: 346; NK: 194; CD14+ mono: 354; CD16+ mono: 98; DC: 38; pDC: 0; Megakaryocyte: 270; Unassigned: 0
\\

& CEL & 526 &
Cytotoxic T: 174; CD4+ T: 160; B: 80; NK: 42; CD14+ mono: 31; CD16+ mono: 21; DC: 0; pDC: 0; Megakaryocyte: 18; Unassigned: 0
\\

& SeqWell & 3,773 &
Cytotoxic T: 1,278; CD4+ T: 566; B: 527; NK: 0; CD14+ mono: 1,255; CD16+ mono: 0; DC: 37; pDC: 26; Megakaryocyte: 38; Unassigned: 46
\\

& SmartSeq2 & 526 &
Cytotoxic T: 193; CD4+ T: 112; B: 79; NK: 50; CD14+ mono: 60; CD16+ mono: 18; DC: 0; pDC: 0; Megakaryocyte: 14; Unassigned: 0
\\

\hline

\multirow{6}{=}{6-batch MHD}
& e8.0\_sce31053 & 5,496 &
Hindbrain glutamatergic: 0; Mesenchyme: 95; Dorsal hindbrain: 1,843; Gliogenic progenitors: 0; VZ progenitors: 0; Caudal: 1,218; Hindbrain roof plate: 1,204; Neural crest: 452; Hindbrain GABAergic: 0; Early ectoderm: 748; Proliferating VZ progenitors: 0; Purkinje cells: 0; Erythrocytes and erythrocyte progenitors: 0; Anterior: 695; Hindbrain: 0; Endothelial cells: 0; Sensory neuron: 0; Extraembryonic endoderm: 342; Endoderm: 187
\\

& e9.0\_sce31053 & 5,103 &
Hindbrain glutamatergic: 4; Mesenchyme: 1,624; Dorsal hindbrain: 0; Gliogenic progenitors: 0; VZ progenitors: 0; Caudal: 38; Hindbrain roof plate: 0; Neural crest: 240; Hindbrain GABAergic: 0; Early ectoderm: 51; Proliferating VZ progenitors: 0; Purkinje cells: 0; Erythrocytes and erythrocyte progenitors: 0; Anterior: 2; Hindbrain: 0; Endothelial cells: 0; Sensory neuron: 639; Extraembryonic endoderm: 0; Endoderm: 0
\\

& e10.0\_sce31053 & 6,606 &
Hindbrain glutamatergic: 62; Mesenchyme: 1,825; Dorsal hindbrain: 0; Gliogenic progenitors: 0; VZ progenitors: 0; Caudal: 0; Hindbrain roof plate: 0; Neural crest: 240; Hindbrain GABAergic: 0; Early ectoderm: 0; Proliferating VZ progenitors: 0; Purkinje cells: 0; Erythrocytes and erythrocyte progenitors: 159; Anterior: 0; Hindbrain: 0; Endothelial cells: 0; Sensory neuron: 0; Extraembryonic endoderm: 0; Endoderm: 0
\\

& E10\_sce27998 & 7,725 &
Hindbrain glutamatergic: 4,758; Mesenchyme: 0; Dorsal hindbrain: 0; Gliogenic progenitors: 0; VZ progenitors: 1,521; Caudal: 0; Hindbrain roof plate: 0; Neural crest: 0; Hindbrain GABAergic: 70; Early ectoderm: 0; Proliferating VZ progenitors: 0; Purkinje cells: 737; Erythrocytes and erythrocyte progenitors: 0; Anterior: 0; Hindbrain: 670; Endothelial cells: 667; Sensory neuron: 0; Extraembryonic endoderm: 0; Endoderm: 0
\\

& E12\_sce27998 & 5,460 &
Hindbrain glutamatergic: 3,710; Mesenchyme: 0; Dorsal hindbrain: 0; Gliogenic progenitors: 36; VZ progenitors: 2; Caudal: 0; Hindbrain roof plate: 0; Neural crest: 0; Hindbrain GABAergic: 715; Early ectoderm: 0; Proliferating VZ progenitors: 703; Purkinje cells: 32; Erythrocytes and erythrocyte progenitors: 0; Anterior: 0; Hindbrain: 0; Endothelial cells: 0; Sensory neuron: 0; Extraembryonic endoderm: 0; Endoderm: 0
\\

& E14\_sce27998 & 3,730 &
Hindbrain glutamatergic: 1,089; Mesenchyme: 0; Dorsal hindbrain: 0; Gliogenic progenitors: 1,739; VZ progenitors: 0; Caudal: 0; Hindbrain roof plate: 0; Neural crest: 0; Hindbrain GABAergic: 49; Early ectoderm: 0; Proliferating VZ progenitors: 70; Purkinje cells: 0; Erythrocytes and erythrocyte progenitors: 0; Anterior: 0; Hindbrain: 0; Endothelial cells: 0; Sensory neuron: 0; Extraembryonic endoderm: 0; Endoderm: 0
\\

\hline

\multirow{2}{=}{2-batch HPMS}
& 0 & 12,720 &
Beta: 4,289; Alpha: 3,595; Ductal: 1,642; Delta: 899; Endothelial: 769; Acinar: 610; Gamma: 418; Stellate: 498
\\

& 1 & 2,961 &
Beta: 796; Alpha: 1,109; Ductal: 462; Delta: 142; Endothelial: 67; Acinar: 103; Gamma: 219; Stellate: 63
\\

\hline
\end{tabularx}
\caption{Cell-type composition imbalance across batches}
\label{tab:dataset_statistics}
\end{table*}

\section{METRICS $-$ DEFINITIONS AND INTERPRETATION} \label{app:complete_scib}

We evaluate data integration performance using the scIB benchmarking framework~\cite{luecken2022benchmarking}, which provides a standardized collection of quantitative metrics to assess the trade-off between batch-effect removal and biological signal preservation. All metrics are computed on the low-dimensional integrated representations produced by each method unless stated otherwise.

\subsection{Biological Conservation Metrics}

Biological conservation metrics quantify how well biologically meaningful structure, in particular cell-type identity and local neighborhood relationships, is preserved after integration.

\paragraph{Clustering-based label agreement (NMI, ARI).}
We assess global agreement between unsupervised clustering results and ground-truth cell-type annotations using normalized mutual information (NMI) and adjusted Rand index (ARI). Let $C$ denote the clustering assignment obtained by applying K-means to the integrated embedding and $Y$ the true cell-type labels. NMI is defined as
\[
\mathrm{NMI}(C, Y) = \frac{2 I(C; Y)}{H(C) + H(Y)},
\]
where $I(\cdot;\cdot)$ denotes mutual information and $H(\cdot)$ denotes entropy. ARI measures pairwise agreement between $C$ and $Y$, adjusted for chance, with values closer to $1$ indicating better recovery of cell-type structure.

\paragraph{Cell-type average silhouette width (ASW).}
The silhouette width for a cell $i$ is defined as
\[
s(i) = \frac{b(i) - a(i)}{\max\{a(i), b(i)\}},
\]
where $a(i)$ is the average distance from $i$ to other cells of the same cell type and $b(i)$ is the minimum average distance to cells of a different cell type. The cell-type ASW is computed by averaging $s(i)$ across all cells and rescaled to $[0,1]$ as
\[
\mathrm{ASW}_{\text{cell}} = \frac{\mathrm{ASW} + 1}{2},
\]
with higher values indicating better separation between cell types.

\paragraph{Isolated-label F1 score.}
The isolated-label F1 score evaluates preservation of rare or batch-specific cell types that appear in only a subset of batches. For each isolated label, precision and recall are computed based on neighborhood consistency, and the F1 score is defined as
\[
\mathrm{F1} = \frac{2 \cdot \mathrm{precision} \cdot \mathrm{recall}}{\mathrm{precision} + \mathrm{recall}}.
\]
Higher values indicate improved conservation of rare cell populations.

\paragraph{Cell-type local inverse Simpson’s index (cLISI).}
cLISI measures the diversity of cell-type labels within local neighborhoods of the integrated embedding. For a cell $i$, cLISI is defined as
\[
\mathrm{cLISI}(i) = \left( \sum_{c} p_{i,c}^2 \right)^{-1},
\]
where $p_{i,c}$ denotes the proportion of neighbors of $i$ belonging to cell type $c$. Lower cLISI values indicate stronger local cell-type purity and better biological conservation.

\subsection{Batch-Effect Removal Metrics}

Batch-effect removal metrics quantify how effectively technical variation across batches or datasets is mitigated while avoiding overcorrection.

\paragraph{Batch average silhouette width (batch ASW).}
Batch ASW is computed analogously to cell-type ASW, but using batch labels instead of cell-type labels. For each cell $i$, the absolute silhouette width $|s(i)|$ is computed and rescaled such that
\[
\mathrm{ASW}_{\text{batch}} = 1 - |s(i)|.
\]
Higher values indicate better batch mixing.

\paragraph{Integration local inverse Simpson’s index (iLISI).}
iLISI evaluates batch diversity within local neighborhoods. For a cell $i$,
\[
\mathrm{iLISI}(i) = \left( \sum_{b} p_{i,b}^2 \right)^{-1},
\]
where $p_{i,b}$ denotes the proportion of neighbors from batch $b$. Higher iLISI values indicate stronger batch mixing.

\paragraph{kBET per label.}
The k-nearest neighbor batch effect test (kBET) evaluates whether batch labels are uniformly distributed within local neighborhoods. In the per-label variant, kBET is applied separately within each cell type. The final score is the average acceptance rate across labels, with higher values indicating reduced batch effects.

\paragraph{Graph connectivity.}
Graph connectivity measures whether cells of the same cell type form a connected component across batches in a k-nearest neighbor graph constructed on the integrated representation. Let $G_c$ denote the subgraph induced by cells of cell type $c$. The connectivity score is defined as the fraction of nodes belonging to the largest connected component of $G_c$, averaged across cell types. Higher values indicate successful integration without fragmenting biological structure.

\paragraph{Principal component regression (PCR).}
PCR quantifies the fraction of variance in the integrated embedding explained by batch labels. Let $Z$ denote the integrated representation and $S$ the batch labels. PCR is computed by regressing principal components of $Z$ on $S$, and the coefficient of determination $R^2$ is averaged across components. Lower PCR values indicate reduced batch-associated variance.

\subsection{Aggregate Scores}

To summarize integration performance, scIB aggregates individual metrics into composite scores. All metrics are first scaled to the unit interval $[0,1]$, with higher values indicating better performance.

The biological conservation score is computed as the arithmetic mean of all biological conservation metrics,
\[
\mathrm{Bio} = \frac{1}{|\mathcal{M}_{\mathrm{bio}}|} \sum_{m \in \mathcal{M}_{\mathrm{bio}}} m,
\]
where $\mathcal{M}_{\mathrm{bio}}$ includes NMI, ARI, cell-type ASW, isolated-label F1, and cLISI.

Similarly, the batch-effect removal score is computed as
\[
\mathrm{Batch} = \frac{1}{|\mathcal{M}_{\mathrm{batch}}|} \sum_{m \in \mathcal{M}_{\mathrm{batch}}} m,
\]
where $\mathcal{M}_{\mathrm{batch}}$ includes batch ASW, iLISI, kBET per label, graph connectivity, and PCR.

Following scIB conventions, the overall integration score is computed as a weighted mean,
\[
\mathrm{Overall} = 0.6 \cdot \mathrm{Bio} + 0.4 \cdot \mathrm{Batch},
\]
placing greater emphasis on biological conservation while still encouraging effective batch-effect removal.

\section{Cumulative Retraining Under Dataset Evolution.} \label{app:wall_clock}

Figure~\ref{fig:figure 3} illustrates the batch mixing performance of \ours for cumulative retraining under protocol arrival. Table~\ref{tab:wall-clock-table} shows wall-clock runtime comparison for this setting.

\begin{figure*}[!t]
\centering
\includegraphics[width=0.7\linewidth]{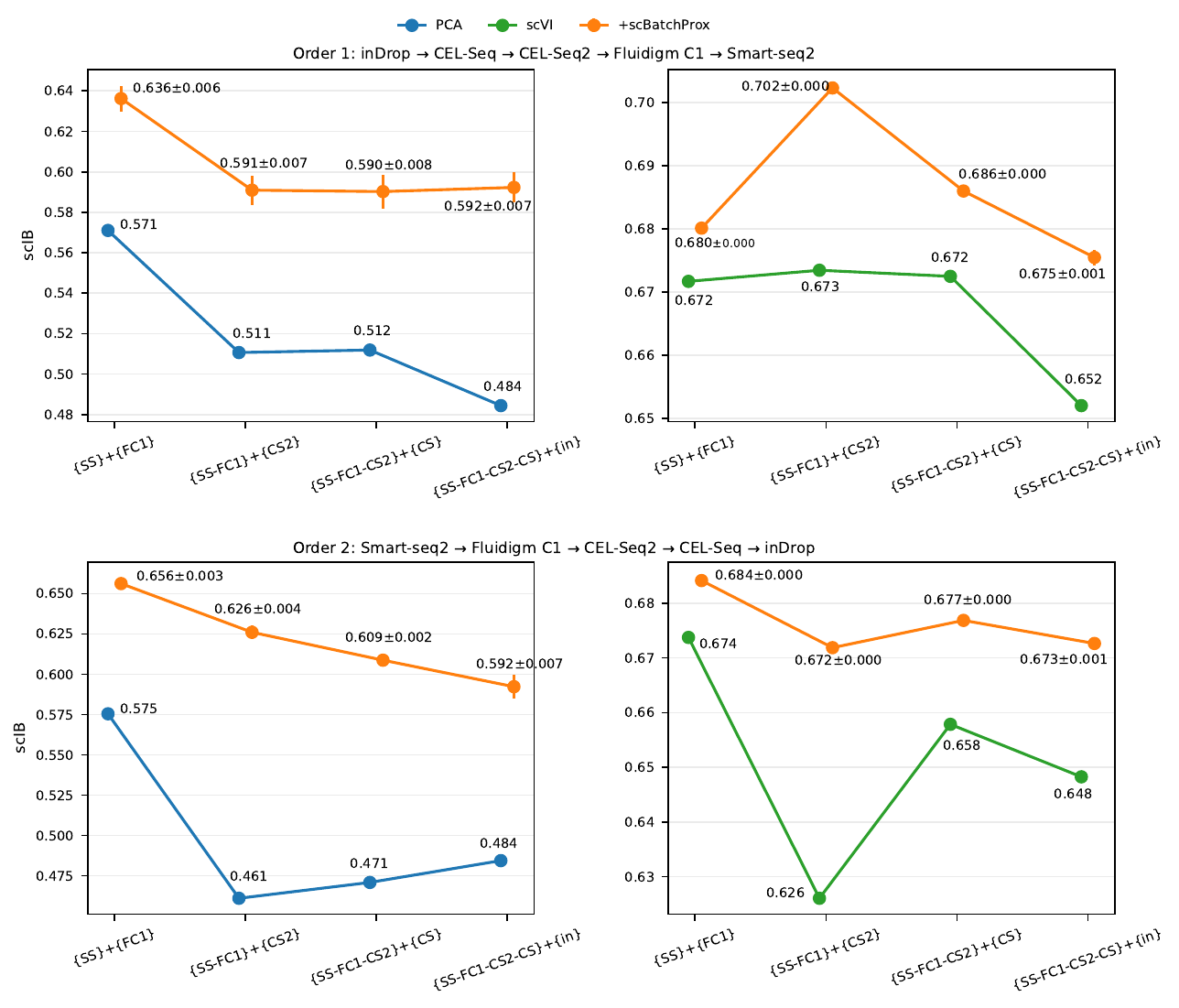}
\caption{
Aggregate scIB scores for cumulative retraining under protocol arrival:
(a) technological progression order
(inDrop $\to$ CEL-Seq $\to$ CEL-Seq2 $\to$ Fluidigm C1 $\to$ Smart-seq2)
and
(b) reverse order
(Smart-seq2 $\to$ Fluidigm C1 $\to$ CEL-Seq2 $\to$ CEL-Seq $\to$ inDrop).
While cumulative retraining is supported by most existing methods, 
it is computationally expensive. \ours consistently improves overall 
embedding quality across both arrival orders.
}
\label{fig:figure 3}
\end{figure*}

\begin{table}[!t]
\centering
\scriptsize
\setlength{\tabcolsep}{2.5pt}
\resizebox{\columnwidth}{!}{
\begin{tabular}{c|l|c|ccc}
\hline
\textbf{Ord.} & \textbf{Method} & \textbf{Step}
& \textbf{Run}
& \textbf{Total}
& \textbf{Cum.} \\
\hline

1 & PCA & 1 & 0.16 & 154.91 & 0.16 \\
1 & PCA & 2 & 0.24 & 236.19 & 0.41 \\
1 & PCA & 3 & 0.45 & 320.40 & 0.86 \\
1 & PCA & 4 & 0.72 & 366.72 & 1.58 \\
\hline

1 & PCA$-$\ours & 1 & 3.35 & 14.93 & 3.35 \\
1 & PCA$-$\ours & 2 & 4.55 & 18.24 & 7.90 \\
1 & PCA$-$\ours & 3 & 3.20 & 18.44 & 11.10 \\
1 & PCA$-$\ours & 4 & 12.20 & 18.04 & 23.30 \\
\hline

1 & scVI & 1 & 154.54 & 154.91 & 154.54 \\
1 & scVI & 2 & 235.65 & 236.19 & 390.20 \\
1 & scVI & 3 & 319.61 & 320.40 & 709.81 \\
1 & scVI & 4 & 365.65 & 366.72 & 1075.45 \\
\hline

1 & scVI$-$\ours & 1 & 11.46 & 14.93 & 11.46 \\
1 & scVI$-$\ours & 2 & 13.58 & 18.24 & 25.04 \\
1 & scVI$-$\ours & 3 & 15.11 & 18.44 & 40.15 \\
1 & scVI$-$\ours & 4 & 5.71 & 18.04 & 45.86 \\
\hline

2 & PCA & 1 & 0.10 & 94.20 & 0.10 \\
2 & PCA & 2 & 0.16 & 155.35 & 0.26 \\
2 & PCA & 3 & 0.26 & 185.90 & 0.52 \\
2 & PCA & 4 & 0.78 & 397.56 & 1.31 \\
\hline

2 & PCA$-$\ours & 1 & 3.75 & 9.08 & 3.75 \\
2 & PCA$-$\ours & 2 & 5.80 & 12.00 & 9.55 \\
2 & PCA$-$\ours & 3 & 5.95 & 10.70 & 15.49 \\
2 & PCA$-$\ours & 4 & 11.14 & 19.56 & 26.63 \\
\hline

2 & scVI & 1 & 93.98 & 94.20 & 93.98 \\
2 & scVI & 2 & 155.00 & 155.35 & 248.98 \\
2 & scVI & 3 & 185.44 & 185.90 & 434.42 \\
2 & scVI & 4 & 396.40 & 397.56 & 830.82 \\
\hline

2 & scVI$-$\ours & 1 & 5.15 & 9.08 & 5.15 \\
2 & scVI$-$\ours & 2 & 6.08 & 12.00 & 11.22 \\
2 & scVI$-$\ours & 3 & 4.63 & 10.70 & 15.86 \\
2 & scVI$-$\ours & 4 & 8.25 & 19.56 & 24.10 \\
\hline

\end{tabular}
}
\caption{Wall-clock runtime comparison under cumulative dataset growth. We report per-step time (Run), total retraining cost (Total), and cumulative runtime (Cum.) across sequential data integration steps for baseline methods (PCA, scVI) and post hoc refinement (\ours). Order 1 corresponds to the technological progression order, and Order 2 corresponds to the reverse order. The reported total time (i.e., full wall-clock time) includes data loading, preprocessing (e.g., data normalization, log1p transformation, etc.), training, and inference.}
\label{tab:wall-clock-table}
\end{table}

\section{Synthetic Heterogeneous Batch Construction in QM8.} \label{app:stat_qm8}

We construct synthetic batches by partitioning molecules based on heavy atom count, thereby grouping molecules by structural complexity. This results in highly imbalanced and heterogeneous batches (Table~\ref{tab:qm8_batches}), reflecting realistic distribution shifts in which certain molecular regimes are underrepresented, with the majority of molecules concentrated in the high-complexity regime. Consequently, batch membership is strongly correlated with molecular structure and feature distribution, creating a challenging setting for post hoc batch correction.

\begin{table}[H]
\centering
\caption{Synthetic batch construction on QM8 based on heavy atom count, resulting in highly imbalanced and heterogeneous groups.}
\small
\begin{tabular}{lcc}
\toprule
\textbf{Batch} & \textbf{\# Molecules} & \textbf{Heavy Atom Range} \\
\midrule
0 & 17    & $\leq 3$ \\
1 & 158   & $4$--$5$ \\
2 & 21{,}611 & $\geq 6$ \\
\bottomrule
\end{tabular}
\label{tab:qm8_batches}
\end{table}

\end{document}